\newif\ifIcassp

\Icassptrue

\ifIcassp
    \documentclass{article}
    \usepackage{spconf}
\else 
    \documentclass[journal,12pt,onecolumn,draftclsnofoot]{IEEEtran}
\fi 

\usepackage{comment}

\usepackage{algorithm}
\usepackage{graphicx}
\usepackage[noend]{algpseudocode}
\usepackage{url}
\usepackage{breqn}
\usepackage{cite}
\usepackage[usenames]{color}
\usepackage{amsfonts}
\usepackage{fancyhdr}
\usepackage{todonotes}
\usepackage{arydshln}
\usepackage{amsmath,amssymb,amsthm}
\usepackage{psfrag, caption,subcaption}
\newtheorem{theorem}{Theorem}

\newtheorem{lemma}{Lemma}

\usepackage{arydshln}
\usepackage{booktabs}
\usepackage{multirow}
\usepackage{adjustbox}

\usepackage{acronym,comment}

\acrodef{cp}[CP]{Conformal prediction}
\acrodef{spci}[SPCI]{sequential predictive conformal inference}
\acrodef{uq}[UQ]{uncertainty quantification}
\acrodef{daps}[DAPS]{Diffused adaptive prediction sets}
\acrodef{enbpi}[EnbPI]{sequential distribution-free ensemble batch
prediction intervals}
\acrodef{gnn-rnn}[GNN-RNN]{graph neural networks with recurrent neural networks}
\acrodef{gnn}[GNNs]{Graph neural networks}
\acrodef{rnn}[RNNs]{Recurrent neural networks}
\acrodef{confgnn}[ConfGNN]{conformalized graph neural network}
\acrodef{dcrnn}[DCRNN]{diffusion convolutional
recurrent neural network}
\acrodef{gconvgru}[GConvGRU]{Chebyshev graph convolutional gated recurrent unit cell}
\acrodef{gnnxmer}[GTs]{graph transformers}
\acrodef{iid}[i.i.d]{independent and identically distributed}

\usepackage{hyperref}

\def\cast{{
   \mathord{
      \hbox to 0em{
         \ooalign{
	   \smash{\hbox{$\ast$}}\crcr
	   \smash{\hskip-1pt\Large\hbox{$\circ$}} }
	 \hidewidth}
      \phantom{\bigcirc}
} }}

\def\bm#1{\mbox{\boldmath $#1$}}

\newcommand{\rT}{^{\top}}

\newcommand{\bds}{\begin {itemize}}
\newcommand{\eds}{\end {itemize}}
\newcommand{\bdf}{\begin{definition}}
\newcommand{\blm}{\begin{lemma}}
\newcommand{\edf}{\end{definition}}
\newcommand{\elm}{\end{lemma}}
\newcommand{\bthm}{\begin{theorem}}
\newcommand{\ethm}{\end{theorem}}
\newcommand{\bprp}{\begin{prop}}
\newcommand{\eprp}{\end{prop}}
\newcommand{\bcl}{\begin{claim}}
\newcommand{\ecl}{\end{claim}}
\newcommand{\bcr}{\begin{coro}}
\newcommand{\ecr}{\end{coro}}
\newcommand{\bquest}{\begin{question}}
\newcommand{\equest}{\end{question}}

\newcommand{\larrow}{{\larrow}}

\newcommand{\argmin}{\ensuremath{\mathrm{arg}\min}}
\newcommand{\argmax}{\ensuremath{\mathrm{arg}\max}}

\newcommand{\cB}{{\ensuremath{\mathcal{B}}}}
\newcommand{\cC}{{\ensuremath{\mathcal{C}}}}

\newcommand{\cE}{{\ensuremath{\mathcal{E}}}}

\newcommand{\cG}{{\ensuremath{\mathcal{G}}}}

\newcommand{\cM}{{\ensuremath{\mathcal{M}}}}

\newcommand{\cV}{{\ensuremath{\mathcal{V}}}}

\newcommand{\ve}{{\ensuremath{{\mathbf{e}}}}}

\newcommand{\mA}{{\ensuremath{\mathbf{A}}}}

\newcommand{\mD}{{\ensuremath{\mathbf{D}}}}

\newcommand{\mH}{{\ensuremath{\mathbf{H}}}}

\newcommand{\mI}{{\ensuremath{\mathbf{I}}}}

\usepackage{latexsym}

\def\IC{\mathbb C}
\def\IN{\mathbb N}
\def\IZ{\mathbb Z}
\def\IR{\mathbb R}

\def\shat{^{\mathchoice{}{}%
 {\,\,\smash{\hbox{\lower4pt\hbox{$\widehat{\null}$}}}}%
 {\,\smash{\hbox{\lower3pt\hbox{$\hat{\null}$}}}}}}

\def\bSigma{{
      \ooalign{
      \smash{\hskip.4pt\raise.4pt\hbox{$\Sigma$}}\vphantom{}\crcr
      \smash{\hskip.7pt\raise.6pt\hbox{$\Sigma$}}\vphantom{}\crcr
      \smash{\hbox{$\Sigma$}}\vphantom{$\Sigma$}}
      \vphantom{\hbox{$\Sigma$}}
      }}
\def\bTheta{{
      \ooalign{
      \smash{\hskip.5pt\raise.5pt\hbox{$\Theta$}}\vphantom{}\crcr
      \smash{\hskip.0pt\raise.1pt\hbox{$\Theta$}}\vphantom{}\crcr
      \smash{\hbox{$\Theta$}}\vphantom{$\Theta$}}
      \vphantom{\hbox{$\Theta$}}
      }}
\def\bDelta{{
      \ooalign{
      \smash{\hskip.4pt\raise.4pt\hbox{$\Delta$}}\vphantom{}\crcr
      \smash{\hskip.7pt\raise.6pt\hbox{$\Delta$}}\vphantom{}\crcr
      \smash{\hbox{$\Delta$}}\vphantom{$\Delta$}}
      \vphantom{\hbox{$\Delta$}}
      }}
\def\bLambda{{
      \ooalign{
      \smash{\hskip.5pt\raise.5pt\hbox{$\Lambda$}}\vphantom{}\crcr
      \smash{\hskip.0pt\raise.1pt\hbox{$\Lambda$}}\vphantom{}\crcr
      \smash{\hbox{$\Lambda$}}\vphantom{$\Lambda$}}
      \vphantom{\hbox{$\Lambda$}}
      }}

\makeatletter

\def\bordermatrix#1{\begingroup \m@th
  \@tempdima 8.75\p@
  \setbox\z@\vbox{%
    \def\cr{\crcr\noalign{\kern2\p@\global\let\cr\endline}}%
    \ialign{$##$\hfil\kern2\p@\kern\@tempdima&\thinspace\hfil$##$\hfil
      &&\quad\hfil$##$\hfil\crcr
      \omit\strut\hfil\crcr\noalign{\kern-\baselineskip}%
      #1\crcr\omit\strut\cr}}%
  \setbox\tw@\vbox{\unvcopy\z@\global\setbox\@ne\lastbox}%
  \setbox\tw@\hbox{\unhbox\@ne\unskip\global\setbox\@ne\lastbox}%
  \setbox\tw@\hbox{$\kern\wd\@ne\kern-\@tempdima\left[\kern-\wd\@ne
    \global\setbox\@ne\vbox{\box\@ne\kern2\p@}%
    \vcenter{\kern-\ht\@ne\unvbox\z@\kern-\baselineskip}\,\right]$}%
  \null\;\vbox{\kern\ht\@ne\box\tw@}\endgroup}
\makeatother

\makeatletter
\def\argmin{\mathop{\operator@font arg\,min}}
\def\argmax{\mathop{\operator@font arg\,max}}
\makeatother

\def\bm#1{\mbox{\boldmath $#1$}}

\newcommand{\bea}{\begin{array}}
\newcommand{\ena}{\end{array}}
\newcommand{\beq}{\begin{equation}}
\newcommand{\enq}{\end{equation}}

\newcommand{\beqa}{\begin{eqnarray}}
\newcommand{\enqa}{\end{eqnarray}}

\newcommand{\beqan}{\begin{eqnarray*}}
\newcommand{\enqan}{\end{eqnarray*}}

\newcommand{\AL}{\begin{enumerate}}
\newcommand{\ALE}{\end{enumerate}}

\def\addots{\mathinner{
    \mkern1mu\raise0pt\vbox{\kern7pt\hbox{.}}
    \mkern2mu\raise4pt\hbox{.}
    \mkern2mu\raise7pt\hbox{.}
    \mkern1mu}}

\def\sddots{\mathinner{
    \mkern.8mu\raise7pt\hbox{.}
    \mkern.8mu\raise4pt\hbox{.}
    \mkern.8mu\raise0pt\vbox{\kern7pt\hbox{.}}
    \mkern1mu}}

\def\saddots{\mathinner{
    \mkern.2mu\raise0pt\vbox{\kern7pt\hbox{.}}
    \mkern.2mu\raise4pt\hbox{.}
    \mkern.2mu\raise7pt\hbox{.}
    \mkern1mu}}

\newcommand{\bc}{{\ensuremath{\mathbf{c}}}}

\newcommand{\be}{{\ensuremath{\mathbf{e}}}}

\newcommand{\bs}{{\ensuremath{\mathbf{s}}}}

\newcommand{\bx}{{\ensuremath{\mathbf{x}}}}
\newcommand{\by}{{\ensuremath{\mathbf{y}}}}

\newcommand{\bH}{{\ensuremath{\mathbf{H}}}}

\def\sqplus{\mathbin{
	{\ooalign{\hfil\raise.3ex\hbox{\scriptsize
	+}\hfil\crcr\mathhexbox274\crcr\mathhexbox275}}
	}} 
\def\sqminus{\mathbin{
	{\ooalign{\hfil\raise.3ex\hbox{\scriptsize
	--}\hfil\crcr\mathhexbox274\crcr\mathhexbox275}}
	}}

\def\IC{{
   \mathord{
      \hbox to 0em{
	 \hskip-4pt
         \ooalign{
	   \smash{\hskip1.9pt\raise2.6pt\hbox{$\scriptscriptstyle |$}}\crcr
	   \smash{\hbox{\rm\sf C}} }
	 \hidewidth}
      \phantom{\hbox{\rm\sf C}}
} }}
\def\IN{
    {\ooalign{
   \smash{\hskip2.2pt\raise1.5pt\hbox{$\scriptscriptstyle |$}}\vphantom{}\crcr
   \hbox{\sf N}
	}}
	} 
\def\IZ{
    {\ooalign{
   \smash{\hskip1.9pt\raise0pt\hbox{$\sf Z$}}\vphantom{}\crcr
   \hbox{\sf Z}
	}}
	} 
\def\IR{
    {\ooalign{
   \smash{\hskip2.2pt\raise1.5pt\hbox{$\scriptscriptstyle |$}}\vphantom{}\crcr
   \smash{\hskip2.2pt\raise3.3pt\hbox{$\scriptscriptstyle |$}}\vphantom{}\crcr
   \hbox{\sf R}
	}}
	} 

\DeclareMathAlphabet{\mathcmb}{OT1}{cmr}{b}{n}

\def\bSigma{\ensuremath{\mathcmb{\Sigma}}}
\def\bLambda{\ensuremath{\mathcmb{\Lambda}}}

\def\bTheta{\ensuremath{\mathcmb{\Theta}}}

\newcommand{\SI}{\begin{indlist}}
\newcommand{\EI}{\end{indlist}}

\newcommand{\DL}{\begin{dashlist}}
\newcommand{\DLE}{\end{dashlist}}

\makeatletter
\def\setboxz@h{\setbox\z@\hbox}
\def\wdz@{\wd\z@}
\def\boxz@{\box\z@}
\def\underset#1#2{\binrel@{#2}%
  \binrel@@{\mathop{\kern\z@#2}\limits_{#1}}}
\def\binrel@#1{\begingroup
  \setboxz@h{\thinmuskip0mu
    \medmuskip\m@ne mu\thickmuskip\@ne mu
    \setbox\tw@\hbox{$#1\m@th$}\kern-\wd\tw@
    ${}#1{}\m@th$}%
  \edef\@tempa{\endgroup\let\noexpand\binrel@@
    \ifdim\wdz@<\z@ \mathbin
    \else\ifdim\wdz@>\z@ \mathrel
    \else \relax\fi\fi}%
  \@tempa
}
\let\binrel@@\relax%
\makeatother

\begin{document}

\title{Conformal Inference for Time Series over Graphs}
\ifIcassp
    \ninept
%

\name{Sonakshi Dua$^\ddag$, Gonzalo Mateos$^\dag$, and Sundeep Prabhakar Chepuri$^\ddag$ 
}
\address{$^\ddag$Indian Institute of Science, Bangalore, India\\
$^\dag$ University of Rochester, Rochester, NY, USA
}
\fi 

\maketitle

\begin{abstract}
Trustworthy decision making in networked, dynamic environments calls for innovative uncertainty quantification substrates in predictive models for graph time series.
Existing conformal prediction (CP) methods have been applied separately to multivariate time series and static graphs, but they either ignore the underlying graph topology or neglect temporal dynamics. To bridge this gap, here we develop a CP-based sequential prediction region framework tailored for graph time series. A key technical innovation is to leverage the graph structure and thus capture pairwise dependencies across nodes, while providing user-specified coverage guarantees on the predictive outcomes.  
We formally establish that our scheme yields an exponential shrinkage in
the volume of the ellipsoidal prediction set relative to its graph-agnostic counterpart. 
Using real-world datasets, we demonstrate that the novel uncertainty quantification framework maintains desired empirical coverage, while achieving markedly smaller (up to 80\%\ reduction) prediction regions than existing approaches.
\end{abstract}

\ifIcassp
\fi 
\maketitle

\section{Introduction} \label{sec:intro} 

Graph time series, a time-evolving multivariate signal defined on the nodes of a graph, arise in diverse application domains such as electricity demand forecasting, network anonmaly detection, traffic flow prediction, and epidemic modeling, to name a few. State-of-the-art methods developed for graph time-series forecasting include \ac{gnn-rnn}~\cite{DCRNN, GConvGRU, ASTGCN} or \ac{gnnxmer}~\cite{chen_strucaware_xmer, GTN}. Such neural models generate point predictions and they typically lack \ac{uq}, which is a critical component for trustworthy and robust decision making in complex, dynamic environments. 

\ac{cp} has emerged as a widely applicable, distribution-free \ac{uq} paradigm in modern machine learning \cite{pmlr-v60-volkhonskiy17a}. Given a pre-trained predictive model, \ac{cp} 
constructs a model-agnostic wrapper using a calibration dataset and a suitable non-conformity score, to estimate efficient prediction sets 
for unseen test data. Interestingly, these prediction sets offer user-prescribed coverage guarantees without imposing any strong distributional assumptions, but require exchangeability between the calibration and test data. Unlike for i.i.d. samples that are customary in statistical learning, the exchangeability assumption is often not tenable for time series and graph data (where sample ordering or node labeling conveys valuable inductive bias).

\noindent\textbf{Related Works.} To circumvent said exchangeability requirement,~\cite{tibshirani_CP_covariate_shift_2020} extends the vanilla \ac{cp} to non-exchangeable data through the notion of weighted exchangeability. Further, \ac{spci}~\cite{SPCI} as well as \ac{enbpi}~\cite{EnbPI} have been proposed for time series, wherein prediction intervals are constructed sequentially without using any calibration data. For the multivariate setting,~\cite{Ellipsoidal_SPCI} developed ellipsoidal sets for prediction regions. \ac{cp} methods have been developed for \emph{static} graph data as well. \ac{daps}~\cite{DAPS} advocates diffusing the node-wise conformity scores over the graph, to incorporate neighborhood information. The \ac{confgnn}~\cite{Conf_GNN} provides coverage guarantees for transductive node classification and regression in static graphs. 
Exchangability holds for base models that are permutation invariant, which means that the model output/non-conformity score is invariant to the ordering of calibration and test nodes on the graph. However, this does not apply to graph time series, where signals evolve dynamically and the train–test split is determined by temporal ordering rather than by node subset selection.\vspace{2pt} 

\noindent\textbf{Proposed Approach and Contributions.} We study time-evolving graph signals for which neither exchangeability nor permutation invariance hold, thus the need arises for a new framework that integrates \ac{cp} with graph-structured time series prediction pipelines.
This paper addresses the gap of constructing valid and sufficiently tight prediction regions for graph time-series forecasting, where both the graph topology and temporal dependencies play a key role. We propose a novel methodology to compute \emph{graph-aware} nonconforming scores, in which the residuals obtained from a base model are first filtered using a graph convolutional operator. Drawing inspiration from~\cite{Ellipsoidal_SPCI}, we form ellipsoids that contain a fraction of the graph-filtered residuals, and a quantile regressor is used to predict the quantiles of the unseen data. At the heart of the proposed approaches lies the homophily assumption, i.e., neighboring nodes tend to have similar residuals. Graph filtering fruitfully incorporates local structural information, 
which provably yields an exponential shrinkage in the volume of the ellipsoidal prediction set relative to its
graph-agnostic counterpart, without sacrificing coverage.

In summary, the major contributions of the paper are as follows.
\begin{itemize}
    \item We develop a novel \ac{cp} framework for graph time series, which exploits homophily via graph-filtered residual scores.
    \item Leveraging graph structure leads to efficiency. We establish exponential shrinkage in the volume of the ellipsoidal prediction set, while guaranteeing a user-prescribed coverage.
    \item Comprehensive tests using real-world data corroborate that our \ac{uq} scheme attains the target coverage, with markedly smaller ellipsoids than a graph-agnostic \ac{cp} baseline. 
\end{itemize}


\section{Preliminaries}\label{sec:prelim}

\noindent\textbf{Notation}. Throughout the paper, we use lowercase (uppercase) boldface letters to denote vectors (respectively, matrices) and calligraphic letters to denote sets. We denote an ellipsoid with radius $r$ and center $\bc\in\mathbb{R}^N$ as 
$\cB(r,\bc,\boldsymbol{\Sigma}) := \{ \bx \in \mathbb{R}^N \:\vert\: (\bx - \bc)\rT \boldsymbol{\Sigma}^{-1} (\bx - \bc) \le r \},$ where $\boldsymbol{\Sigma}$ is a positive definite matrix that determines the shape of the ellipsoid. The volume of the ellipsoid follows $\mathrm{Vol}\,\big(\cB(r,\bc,\boldsymbol{\Sigma})\big) \propto r^{N/2} \, \sqrt{\det(\boldsymbol{\Sigma})}.$ Next, we provide a brief background of conformal inference, particularly for regression tasks.\vspace{2pt}

\noindent\textbf{Conformal Inference}. \ac{cp} is a distribution-free method to build a wrapper around a pre-trained model that forms prediction sets to guarantee finite-sample coverage. Given a base model (e.g., a classifier or regressor) and a non-conformity score function, \ac{cp} constructs a region $\cC(1-\alpha)$ that contains the true prediction with a user-specified probability \(1-\alpha\), i.e., for a feature $\bx$ and true label $\by$, we have the following conditional coverage guarantee
\[
\mathbb{P}\,\big(\by \in {\cC}(1-\alpha)\:\vert\: \bx\big) \geq 1-\alpha. 
\]
The size of the prediction set determines the \emph{inefficiency} of the model, with a smaller prediction set size indicating more confidence. The uncertainty prediction sets, as small as possible, are constructed to include all potential observations whose non-conformity scores are within the empirical quantiles of the calibration data, which is assumed to be exchangeable with the test data. Calibration data can be held out from the training set, or, in the case of sequential data, the training data can be repurposed with temporal windowing.\vspace{2pt}

\noindent\textbf{Graph Filters.} Consider an undirected graph $\cG = (\cV,\cE)$,  where $\cV$ is the set of $N$ nodes and $\cE\subseteq\cV\times\cV$ is the set of edges. The graph structure is captured by an $N \times N$ symmetric matrix having non-zero $(i,j)$-th entry whenever nodes $i$ and $j$ are connected. Examples of such matrices are the adjacency matrix $\mA \in \mathbb{R}_+^N$ and the normalized adjacency matrix $\mD^{-1}\mA \in \mathbb{R}_+^N$, to name a few. Here, $\mD$ is the diagonal nodal degree matrix. Note that the eigenvalues of $\mD^{-1}\mA$, denoted by $\lambda_i$, are bounded as $\lambda_i \in [0,1]$. Henceforth, we consider $\mD^{-1}\mA$ to represent $\cG.$

A \emph{graph convolutional filter} is defined as the matrix polynomial  
$\smash{\mH = \sum_{l=0}^{L-1}\tau_l(\mD^{-1}\mA)^l}$, where $\smash{\{\tau_l\}_{l=0}^{L-1}}$ denote the filter coefficients and $L$ is the filter order. In this paper, we adopt a simple first-order graph filter of the form $\mH = (1-\tau) \mI + \tau \mD^{-1}\mA$, with $\tau \in [0,1].$  Let $\be \in \mathbb{R}^N$ be a \emph{graph signal}, i.e., a signal supported in $\mathcal{V}$ where $e_i$ is the value at node $i\in\cV$.
Then, we can diffuse $\be$ over the graph by computing $\mH \be$ via one-hop local aggregations.

\section{Problem Statement}\label{sec:prob_statement}

We consider a forecasting problem from graph time-series data. The graph $\cG$ is static, i.e., the graph structure (nodes and edges) remains fixed, but the graph signals (or nodal features) evolve over time. 

Let us denote the graph time-series data at time $t$ as $\cG_t = (\bx_t,\cG)$, which arrives sequentially for $t\geq 1$. Here, $\bx_t \in \mathbb{R}^{N}$ contains the graph signal and $\by_t \in \mathbb{R}^{N}$ is the prediction target. For instance, for a one-step ahead predictor, we use $\by_t = \bx_{t+1}.$ We are given a sequential neural model, denoted by $\cM$, whose goal is to predict $\by_t$ from $(\bx_t,\cG)$ at each $t$, i.e., it provides a \emph{point estimate} $\hat{\by}_t = \cM(\cG_t)$. For instance, $\cM$ could be a \ac{gnn-rnn} or \ac{gnnxmer}.
Given a user-specified miscoverage level $\alpha \in [0,1]$, the main goal of the paper is to estimate a prediction set ${\cC}_{t-1}(\alpha)$ that contains the true target $\by_t$ with probability of at least $1-\alpha$, that is
\[
\mathbb{P}(\by_t \in \cC_{t-1}(\alpha)) \geq 1-\alpha, \, \text{for each \,\,} t.
\]
This is referred to as marginal coverage. Specifically, given training data $\{\cG_t,\by_t\}_{t=1}^T$, the goal is to construct prediction regions $\cC_{t}(\alpha)$ sequentially for all $t \geq T+1$. In addition to guaranteeing coverage, it is crucial to construct efficient prediction regions (with smallest possible volume), and to do so, we exploit the graph structure.

\begin{figure}
    \centering
    \includegraphics[width=0.8\columnwidth]{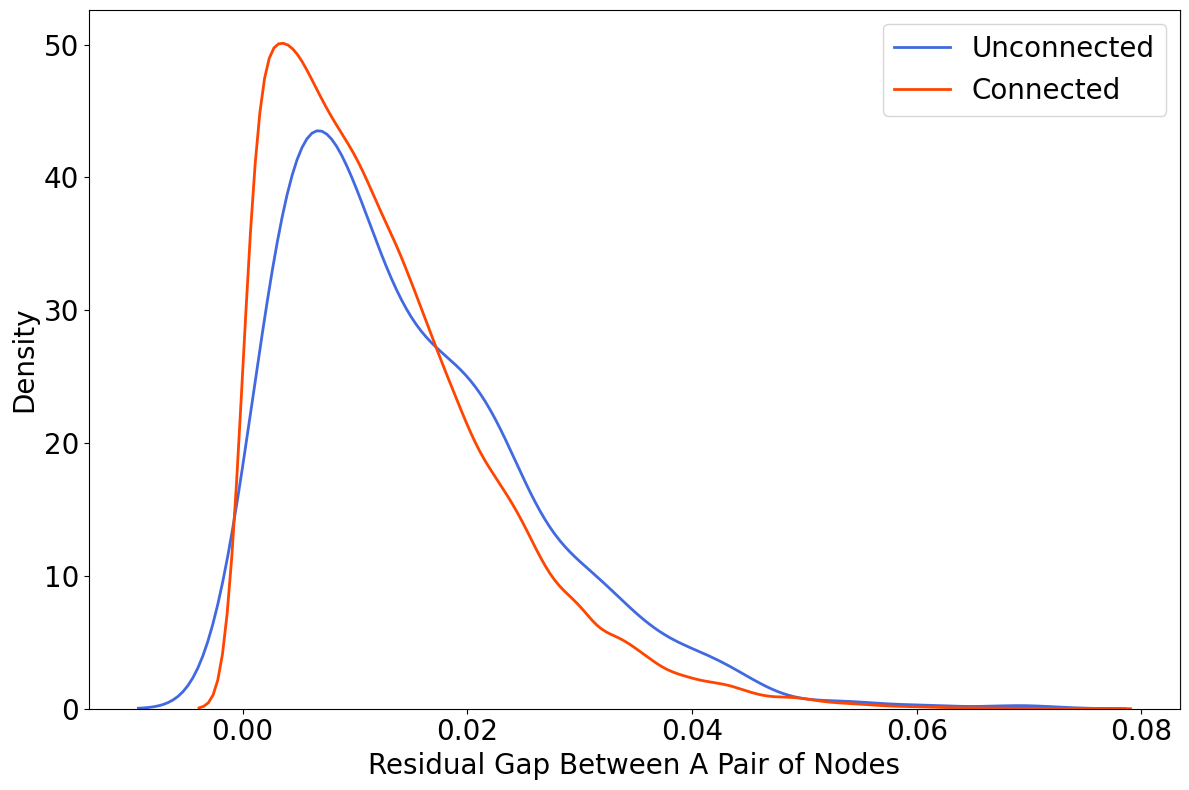}
    \caption{Residuals are smooth over the graph. Connected nodes have a smaller gap in the residuals than unconnected nodes. The plot is generated using prediction residuals from a model $\cM$ on the Wiki Maths dataset (see also Section \ref{sec:num_exp}), averaged across all timesteps.}
    \label{fig: correlation b/w residuals}
\end{figure}

\section{Graph-aware Ellipsoidal Prediction Sets}
In this section, we describe our construction of prediction sets guaranteeing marginal coverage and provable efficiency.
We were inspired by~\cite{Ellipsoidal_SPCI}, which we extend to the graph setting. In Section \ref{sec:theo_analysis} we show that, by exploiting the graph structure, the volume of the ellipsoid shrinks relative to a comparable graph-agnostic baseline.

\subsection{Graph-aware Nonconformity Score}

Consider a pre-trained graph time-series prediction model $\cM$. Suppose we have $T$ training samples $\{\cG_t, \by_t\}$ for $t=1,\ldots,T$. We can compute the prediction residuals as
\begin{equation}
    \boldsymbol{\varepsilon}_t (\by_t) := \by_t - \cM(\cG_t), \quad t=1,\dots,T.
\end{equation}
Assuming homophily, i.e., prediction targets are smooth over the graph and the model exploits this property, the prediction residuals will also be smooth along the edges $\cE$. See Fig.~\ref{fig: correlation b/w residuals}, where we can see that residuals are smaller on connected nodes than on disconnected pairs (a similar observation for static graph data can be found in~\cite{Jia2020kdd}).
To capitalize on this favorable structure, we filter the residuals using a graph filter $\mH$. The filtered residuals are given by
\begin{equation}
    {\be}_t (\by_t)= \mH \boldsymbol{\varepsilon}_t = \left[(1-\tau) \mI + \tau \mD^{-1}\mA\right] \boldsymbol{\varepsilon}_t, \quad t=1,\dots,T.
\end{equation}

Consider the mean residual vectors $\bar{\boldsymbol{\varepsilon}}= \mathbb{E}[\boldsymbol{\varepsilon}_t]$ and $\bar{\be}= \mathbb{E}[\be_t]$. Accordingly, the covariance matrices are $\boldsymbol{\Sigma} = \mathbb{E}[(\boldsymbol{\varepsilon}_t-\bar{\boldsymbol{\varepsilon}})(\boldsymbol{\varepsilon}_t-\bar{\boldsymbol{\varepsilon}})\rT]$ and $\boldsymbol{\Sigma}_\cG = \mathbb{E}[(\be_t-\bar{\be})(\be_t-\bar{\be})\rT] = \mH \boldsymbol{\Sigma} \mH$. 
We define a \emph{graph-aware} nonconformity score function using the (squared) Mahalanobis distance as~\cite{Ellipsoidal_SPCI}
%
\begin{equation}\label{eq:scores}
    s_t(\by_t) := (\be_t - \bar{\be} )\rT\boldsymbol{\Sigma}_\cG^{-1} (\be_t - \bar{\be}).
\end{equation}
The scores $\{s_t\}$ can be used to compute an empirical $\alpha$-quantile radius of the ellipsoid that contains the fraction of residual vectors $\{\be_t\}$. In practice, the ensemble covariance matrices and mean vectors are replaced with their respective sample estimates.

\subsection{Ellipsoidal Uncertainty Sets}\label{ssec:unc_sets}

To predict the quantiles of future unseen nonconformity scores (say, the next score $p_{t'} = s_{t'+1}$), we use quantile regression~\cite{quantile_regr} 
on a temporal window of non-conformity scores $\bs_{t'} = [s_{t'+w-1},\ldots, s_{t'}]\rT$ 
obtained from training samples $t' = 1,\ldots,T-w$, for some window length $w$~\cite{SPCI}. That is, we fit a quantile regressor $Q_t(1-\alpha)$ using samples $\{\hat{p}_{t'} = Q_{t'}(1-\alpha),p_{t'}\}_{t'=1}^{T-w}$.  All in all, we have~\cite{Ellipsoidal_SPCI}
\begin{align}\label{eq:pred_set}
C_t(\alpha) &= \left\{\by_t \:\vert\: s_t(\by_t) \leq Q_t(1-\alpha)\right\}    \notag \\
&= \cM(\cG_t) + \cB(\sqrt{Q_t(1-\alpha)}, \bar{\be}, \boldsymbol{\Sigma}_\cG).
\end{align}
To implement the regression step, one can use off-the-shelf quantile regressors that optimize the pinball loss or quantile random forests.

\section{Theoretical Analysis}\label{sec:theo_analysis}

Here we offer theoretical results on coverage and volume shrinkage.

\subsection{Conditional Coverage}


Our construction enjoys the same coverage guarantees as~\cite{Ellipsoidal_SPCI}, which we reproduce here in the interest of self-containment. 
\begin{theorem}~\cite[Corollary 4.14]{Ellipsoidal_SPCI}.
Assume the true covariance matrix $\boldsymbol{\Sigma}_\cG$ is known and positive definite with minimum eigenvalue at least $\lambda>0$, the filtered residuals 
$\{\ve_t\}$ are i.i.d. over time, and the CDF of the true nonconformity score is Lipschitz. Then, we have
\begin{equation}\tag{4}
\begin{split}
\left|\mathbb{P}\!\big(\mathbf{y}_{T+1}\in\mathcal{C}_{T+1}(\alpha)\mid \mathcal{G}_{T+1}\big)
-(1-\alpha)\right| \\
\le 12\sqrt{\frac{\log(16T)}{T}}
&+ L\left(\frac{\delta_T}{\sqrt{\lambda}}+\delta_T\right).
\end{split}
\end{equation}
where $T$ is the training data size, $\delta_T$ is the bound on the residual error, and $L$ depends on the Lipschitz constant.
\end{theorem}
Notice that the effect of $\mathcal{G}$ is captured in the constants $\delta_T$ and $\lambda$.

\subsection{Ellipsoid Volume Shrinkage}  

Let us define $Q_t'(1-\alpha)$ as the $(1-\alpha)$-quantile computed from the graph-agnostic scores $\psi_t(\by_t) := (\boldsymbol{\varepsilon}_t - \bar{\boldsymbol{\varepsilon}} )\rT\boldsymbol{\Sigma}^{-1} (\boldsymbol{\varepsilon}_t - \bar{\boldsymbol{\varepsilon}} )$ [cf. \eqref{eq:scores}], using the same windowing procedure described at the end of Section \ref{ssec:unc_sets}. For convinience, let $\cB_\cG := \cB(\sqrt{Q_t(1-\alpha)}, \bar{\be}, \boldsymbol{\Sigma}_\cG)$ and $\cB' := \cB(\sqrt{Q_t'(1-\alpha)}, \bar{\boldsymbol{\varepsilon}}, \boldsymbol{\Sigma})$ be the resulting graph-aware and graph-agnostic ellipsoids, respectively.

Under mild assumptions on the graph-agnostic and graph-aware quantiles, i.e., we want to approximately include the same fraction of points in 
the respective ellipsoids, we show that the volume of the prediction sets in \eqref{eq:pred_set} shrinks exponentially in the filter coefficient. 
\begin{theorem}\label{th:vol_shrinkage} Suppose $\frac{Q_t(1-\alpha)}{Q_t'(1-\alpha)} \approx 1$. Then we have
\begin{equation*}
\mathrm{Vol}(\cB_\cG) \leq e^{-\eta \tau} \rm{Vol}(\cB'),
\end{equation*}
for some positive number $\eta$, where $\tau$ is the graph filter coefficient.
\end{theorem}
\begin{proof}
We can write the log ratio of the volumes of two ellipsoids as
\begin{equation}\label{eq:log_ratio}
\log \left[ \frac{\mathrm{vol}(\cB_\cG)}{\mathrm{vol}(\cB')} \right] 
=  \log \det(\bH) + \frac{N}{2} \log \left(\frac{Q_t(1-\alpha)}{Q_t'(1-\alpha)}\right). 
\end{equation}
When the quantiles before and after graph filtering are approximately the same, i.e., $\frac{Q_t(1-\alpha)}{Q_t'(1-\alpha)} \approx 1$, the second term in the right-hand-side of \eqref{eq:log_ratio} vanishes. Thus, to establish ellipsoid volume shrinkage we must show that $\det(\bH) < 1$. To that end, we have 
\begin{align*}
    \log \det(\bH) &= \log \det((1-\tau)\mI + \tau \mD^{-1}\mA) \notag \\
    &=  \sum\limits_{i=1}^N \log (1 - \tau (1- \lambda_i) ) \notag \\
    &\leq -\tau \sum\limits_{i=1}^N (1-\lambda_i)
\end{align*}
as $\log (1-x) \leq -x$. Since $\lambda_i \in [0,1]$, then $\log \det(\bH) \leq -\tau \eta$ or equivalently $\det(\bH) \leq e^{-\eta\tau}$, where $\smash{\eta := \sum_{i=1}^N (1-\lambda_i) >0.}$
\end{proof}
\begin{table}[t]
\centering
\begin{tabular}{l c c}
\toprule
\textbf{Datasets} & nodes & timestamps \\
\midrule
Chickenpox Hungary &  20 & 513\\
\midrule
MontevideoBus & 675 & 740\\
\midrule
Wiki Maths & 1068  & 717\\
\bottomrule
\end{tabular}
\caption{Dataset statistics.}
\label{tab:datasets}
\end{table}
In a nutshell, Theorem \ref{th:vol_shrinkage} asserts that the volume of the graph-aware ellipsoid $\cB_\cG$ shrinks exponentially with the filter coefficient $\tau$ and the spectrum of the underlying graph $\eta$, wherein $\tau$ can be appropriately designed to get a minimum volume ellipsoid.

\begin{table*}[h]
\centering
\begin{tabular}{l l c c c}
\toprule
\textbf{Method} & \textbf{Metric} & \textbf{Wiki Maths} & \textbf{MontevideoBus} & \textbf{Chickenpox Hungary} \\
\midrule
\multirow{2}{*}{\textbf{Graph-agnostic}~\cite{Ellipsoidal_SPCI}} 
    & Coverage & $0.903 \pm 0.006$ & $0.91 \pm 0.0122$ & $0.89 \pm 0.0167$\\
    & Volume   & $5.19\times10^{3} \pm 2010.298$ & $3.09\times10^{3} \pm 247.325$ & $2.74\times10^{2} \pm 14.842$\\
\midrule
\vspace{1mm}
\multirow{2}{*}{\textbf{Graph-aware (proposed)}}
    & Coverage & $0.897 \pm 0.010$ & $0.912 \pm 0.008$ & $0.89 \pm 0.018$ \\
    & Volume   & $\mathbf{1.46\times10^{3} \pm 135.769}$ & $\mathbf{1.56\times10^{3} \pm 880.327}$ & $\mathbf{1.25\times10^{2} \pm 70.851}$ \\
\bottomrule
\end{tabular}
\caption{Empirical coverage and inefficiency for one-step-ahead prediction, for $\alpha = 0.1$ and window length $w=10$. Results are averaged over five runs with a train/test split of $0.7$.}
\label{tab:results_alpha0.1}
\end{table*}
\begin{table*}[h]
\centering
\begin{tabular}{l l c c c}
\toprule
\textbf{Method} & \textbf{Metric} & \textbf{Wiki Maths} & \textbf{MontevideoBus} & \textbf{Chickenpox Hungary} \\
\midrule
\multirow{2}{*}{\textbf{Graph-agnostic}~\cite{Ellipsoidal_SPCI}} 
    & Coverage & $0.954 \pm 0.005$ & $0.948 \pm 0.004$ & $0.916 \pm 0.011$ \\
    & Volume   & $8.51\times10^{3} \pm 1458.303$ & $1.406\times10^{4} \pm 205.985$ & $1.6\times10^{2} \pm 14.153$ \\
\midrule
\multirow{2}{*}{\textbf{Graph-aware (proposed)}} 
     & Coverage & $0.952 \pm 0.004$ & $0.952 \pm 0.008$ & $0.924 \pm 0.005$ \\
    & Volume   & $\bm{2.04\times10^{3} \pm 238.579}$ & $\bm{2.7\times10^{3}  \pm 112.472}$ & $\bm{1.29\times10^{2} \pm 20.144}$ \\
\bottomrule
\end{tabular}
\caption{Empirical coverage and inefficiency for one-step-ahead prediction, for $\alpha = 0.05$, window length $w=10$. Results are averaged over five runs with a train/test split of 0.7.}
\label{tab:results_alpha0.05}
\end{table*}
\begin{table*}[!htbp]
\centering
\begin{adjustbox}{max width=\textwidth}
\begin{tabular}{l
    r @{ $\pm$ } l r @{ $\pm$ } l
    r @{ $\pm$ } l r @{ $\pm$ } l
    r @{ $\pm$ } l r @{ $\pm$ } l}
\toprule
\multicolumn{1}{l}{\textbf{Wiki Maths Dataset}} 
  & \multicolumn{4}{c}{$r= 1$}
  & \multicolumn{4}{c}{$r= 5$} 
  & \multicolumn{4}{c}{$r= 10$} \\
\cmidrule(lr){2-5}\cmidrule(lr){6-9}\cmidrule(lr){10-13}
\textbf{Method} 
  & \multicolumn{2}{c}{Coverage} & \multicolumn{2}{c}{Volume} 
  & \multicolumn{2}{c}{Coverage} & \multicolumn{2}{c}{Volume} 
  & \multicolumn{2}{c}{Coverage} & \multicolumn{2}{c}{Volume} \\
\midrule
\textbf{Graph-agnostic}~\cite{Ellipsoidal_SPCI} 
  & 0.903 & 0.006  & $5.19\times10^{3}$ & 2010.29 
  & 0.889 & 0.012  & $1.20\times10^{4}$ & 754.91  
  & 0.875 & 0.008  & $9.40\times10^{3}$ & 5218.74 \\
\textbf{Graph-aware (proposed)}  
  & 0.897 & 0.010  & $\bm{1.46\times10^{3}}$ & \bm{135.77} 
  & 0.885 & 0.0129 & $\bm{2.40\times10^{3}}$ & \bm{694.09}  
  & 0.867 & 0.002  & $\bm{3.50\times10^{3}}$ & \bm{595.01} \\
\bottomrule
\end{tabular}
\end{adjustbox}
\caption{Coverage and volume for the Wiki Maths dataset, for different multi-step prediction settings. Target coverage is $1-\alpha = 0.9$.}
\label{tab:multi-step}
\end{table*}

\begin{table*}[!htbp]
\centering
\begin{adjustbox}{max width=\textwidth}
\begin{tabular}{l
    r @{ $\pm$ } l r @{ $\pm$ } l
    r @{ $\pm$ } l r @{ $\pm$ } l
    r @{ $\pm$ } l r @{ $\pm$ } l}
\toprule
\multicolumn{1}{l}{\textbf{Wiki Maths Dataset}} 
  & \multicolumn{4}{c}{$w=10$}
  & \multicolumn{4}{c}{$w=50$} 
  & \multicolumn{4}{c}{$w=100$} \\
\cmidrule(lr){2-5}\cmidrule(lr){6-9}\cmidrule(lr){10-13}
\textbf{Method} 
  & \multicolumn{2}{c}{Coverage} & \multicolumn{2}{c}{Volume} 
  & \multicolumn{2}{c}{Coverage} & \multicolumn{2}{c}{Volume} 
  & \multicolumn{2}{c}{Coverage} & \multicolumn{2}{c}{Volume} \\
\midrule
\textbf{Graph-agnostic}~\cite{Ellipsoidal_SPCI} &0.903 & .006 & $5.19\times10^{3}$ & 2010.298 & 0.892 & 0.005 & $4.80\times10^{3}$ & 1135.54  & 0.885 & 0.014 & $2.85\times10^{3}$ & 1260.33\\
\textbf{Graph-aware (proposed)}  &0.897& 0.01 & $\bm{1.46\times10^{3}}$& \bm{135.769} & 0.896 & 0.009 & $\bm{1.27\times10^{3}}$ & \bm{182.67}  & 0.886 & 0.002 & $\bm{1.22\times10^{3}}$ & \bm{345.89} \\
\bottomrule
\end{tabular}\end{adjustbox}
\caption{Coverage and volume for the Wiki Maths dataset, for three different window lengths $w$. Target coverage is $1-\alpha = 0.9$.}
\label{tab: different window lengths}
\end{table*}

\section{Numerical Experiments}\label{sec:num_exp}

We conduct comprehensive numerical experiments on real-world datasets to demonstrate the advantages of employing the novel \ac{cp} method over its graph-agnostic counterpart~\cite{Ellipsoidal_SPCI}, particularly with regard to efficiency improvements. Furthermore, we demonstrate that the proposed graph-aware approach achieves the desired coverage guarantees with a smaller ellipsoid prediction set volume.\vspace{2pt}

\noindent\textbf{Datasets.} We evaluate the performance of our method on three real-world graph time-series datasets~\cite{datasets_pytorch}. These datasets span diverse network types and tasks, including traffic forecasting (MontevideoBus), web traffic prediction (Wiki Maths), and epidemic modeling (Chickenpox Hungary). Additional dataset statistics are summarized in Table \ref{tab:datasets}. \vspace{2pt}
%

\noindent\textbf{Graph Time Series Prediction Model.}
For the Wiki Maths and MontevideoBus datasets we employed the \ac{dcrnn}~\cite{DCRNN}, while for the Chickenpox Hungary dataset, a \ac{gconvgru}~\cite{GConvGRU} was used to obtain the point predictions. To fully utilize the training data, we trained 15 bootstrap models.\vspace{2pt}

\noindent\textbf{Results and Discussion.} To systematically evaluate robustness across different levels of uncertainty and forecasting horizons, we perform the following experiments: 
\begin{enumerate}
    \item \textit{Coverage for different significance level:} We consider $\alpha \in \{0.05, 0.1\}$ to study the trade-off between coverage and ellipsoid volume. The results in Tables~\ref{tab:results_alpha0.1} and~\ref{tab:results_alpha0.05} show that the proposed method maintains coverage across different confidence levels. We observe that the prediction set volumes from the proposed method are smaller in both cases, but the ellipsoid prediction regions become larger as the significance level decreases, as expected.
    \item \textit{Multiple window lengths:} Recall that the window length $w$ is the number of past residuals used to predict the quantile of the future residual. We use $w \in \{10, 50, 100\}$ to analyze how sensitive the method is to temporal context. From Table~\ref{tab: different window lengths}, we observe that coverage stays similar with increasing window lengths, while the volume decreases as $w$ increases. This behavior can be attributed to the quantile predictor having access to more data and the fact that the residuals evolve smoothly over time without abrupt changes.
    \item \textit{One-step-ahead vs. multi-step-ahead prediction:} Multi-step prediction refers to predicting multiple future values beyond the immediate next step. Given past observations $\{\by_t\}_{t=1}^T$, an $r$-step predictor estimates $\by_{T+r}$ for $r > 1$. In practice, the models provide multiple future predictions at once, i.e., for different values of $r$. Results in Table \ref{tab:multi-step} indicate that coverage dimishes as the number $r$ of future steps to be predicted increases. 
\end{enumerate}
In summary, the proposed \ac{cp} method achieves the desired coverage for one-step-ahead prediction for both $\alpha = 0.05$ and $\alpha = 0.1$, while yielding substantially smaller ellipsoid volumes, up to 80\%\ reduction compared to the graph-agnostic baseline~\cite{Ellipsoidal_SPCI}. Our approach consistently outperforms the graph-agnostic counterpart across different window lengths, although the relative gain diminishes as $w$ increases. In the case of multi-step-ahead prediction, we observe a dip in coverage, indicating the need for further methodological improvements to address longer forecasting horizons.


\section{Conclusions}
In this work, we have developed a conformal inference framework for graph time-series forecasting. Given a base predictive model (for instance a \ac{gnn-rnn}), our sequential framework produces efficient ellipsoids that can be used for uncertainty quantification with guaranteed user-specified coverage. In order to account for the graph structure in the nonconformity scores, we exploit the homophilic nature of the residuals via (low-pass) graph convolutional filtering. This mechanism results in provably exponentially smaller (hence, efficient) uncertainty quantifying ellipsoids relative to their graph-agnostic counterparts, while guaranteeing the desired coverage. Specifically, we show that the volume of the ellipsoid shrinks exponentially with the filter coefficient and spectrum of the underlying graph. We conducted experiments on several real-world graph time-series datasets in order to corroborate the benefits of incorporating the graph structure in the proposed nonconformity scores. Since here we compute the prediction regions from scratch at each time instant, it would be an interesting future research direction to adapt the ellipsoids through a recursive update procedure.


\bibliographystyle{IEEEbib}
\bibliography{references}

\end{document}